%% file: acl_latex.tex
\documentclass[11pt]{article}

\PassOptionsToPackage{table}{xcolor}
\usepackage[final]{acl}

\usepackage{times}
\usepackage{latexsym}
\usepackage{amsmath,amssymb}
\usepackage{amsthm}
\newtheoremstyle{acltheorem}
  {\topsep}{\topsep}
  {\itshape}
  {0pt}
  {\bfseries}{.}
  {.5em}
  {\thmname{#1}~\thmnumber{#2}\thmnote{ (#3)}}
\theoremstyle{acltheorem}
\newtheorem{proposition}{Proposition}

\usepackage[T1]{fontenc}
\usepackage[utf8]{inputenc}
\usepackage{microtype}
\usepackage{inconsolata}
\usepackage{graphicx}
\usepackage{algorithm}
\usepackage{tabularx}
\usepackage{algpseudocode}
\usepackage{booktabs}
\usepackage{afterpage}
\usepackage{multirow}
\usepackage{makecell}
\usepackage{array}
\usepackage{xcolor}
\definecolor{gainGreen}{RGB}{0,128,0}
\definecolor{lossRed}{RGB}{200,0,0}
\definecolor{refRowGray}{RGB}{240,240,240}
\definecolor{methodRowBlue}{RGB}{232,240,250}
\definecolor{spectralRowGold}{RGB}{255,245,220}
\definecolor{definitionBoxGray}{RGB}{246,247,248}
\definecolor{definitionBoxBorder}{RGB}{190,196,202}
\newcommand{\up}[1]{{\small\color{gainGreen}(+#1)}}
\newcommand{\dn}[1]{{\small\color{lossRed}(-#1)}}

\usepackage{subcaption}
\usepackage{enumitem}
\usepackage{url}

\graphicspath{{figures/}}

\newcommand{\shortmethod}{WIDER}
\newcommand{\erank}{\operatorname{erank}}

\newcommand{\sg}{\operatorname{sg}}
\newcommand{\definitionbox}[1]{%
\begin{center}
\begingroup
\setlength{\fboxsep}{4pt}%
\setlength{\fboxrule}{0.45pt}%
\fcolorbox{definitionBoxBorder}{definitionBoxGray}{%
\begin{minipage}{0.86\columnwidth}
#1
\end{minipage}}%
\endgroup
\end{center}}

\title{Think Wider: Mitigating Latent Rank Collapse in Implicit Chain-of-Thought Reasoning}

\author{Yuwen Hao \quad Menglin Yang\thanks{~Corresponding author.} \\
  The Hong Kong University of Science and Technology (Guangzhou) \\
  \texttt{yhao481@connect.hkust-gz.edu.cn} \quad \texttt{menglinyang@hkust-gz.edu.cn} \\
}

\begin{document}
\maketitle
\input{sections/sec0_abstract}

\begin{figure}[t!]
\centering
\includegraphics[width=\columnwidth]{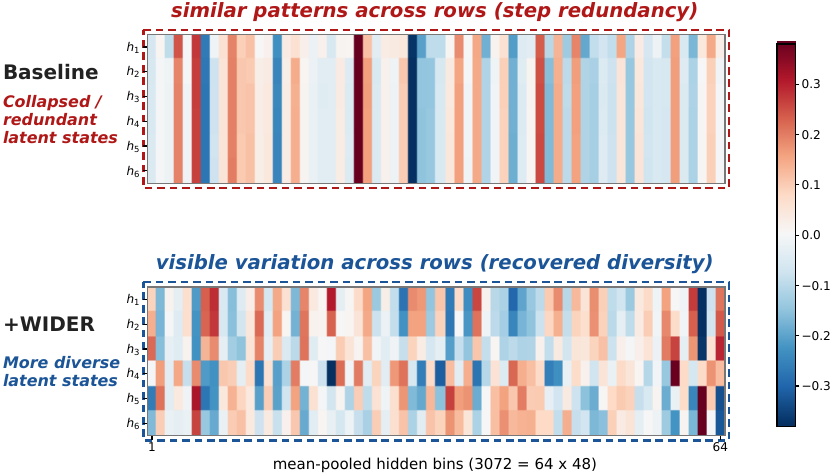}
\caption{Latent rank collapse in implicit CoT. Each panel visualizes one latent
trajectory: rows are six latent steps, and columns are mean-pooled hidden
dimensions. The baseline (\textbf{top}) repeats similar patterns across steps, yielding
a low-rank trajectory (effective rank $1.49$). \shortmethod{} regularization (\textbf{bottom})
suppresses the shared dominant mode and restores stepwise variation (effective
rank $3.81$).}
\label{fig:rank_collapse}
\end{figure}

\input{sections/sec1_introduction}

\input{sections/sec2_related_work}
\input{sections/sec3_rank_collapse}
\input{sections/sec4_method}

\input{sections/sec5_experiment}
\input{sections/sec6_discussion}

\input{sections/sec7_conclusion}
\input{sections/sec8_limitations}
\input{sections/sec9_ethics}
\input{sections/sec10_acknowledgments}

\bibliography{custom}
  
\clearpage
\appendix
\input{appendix/appendix}

\end{document}

%% file: sections/sec0_abstract.tex
\begin{abstract}
Chain-of-thought (CoT) reasoning improves the reasoning ability of large language models by introducing intermediate computation, but explicit rationales increase decoding length, latency, and context cost. Implicit CoT offers a more efficient alternative by moving intermediate reasoning into continuous latent states. However, latent reasoning can be unstable: successive latent states may become overly similar and collapse toward a shared dominant direction, reducing the diversity of the reasoning trajectory. In this work, we identify \emph{latent rank collapse} and propose \textbf{\shortmethod{}}, a lightweight spectral regularizer for implicit CoT. During training, \shortmethod{} estimates the shared direction of each latent trajectory and penalizes projections onto this direction, encouraging latent states to span a broader representational subspace. The method is plug-and-play and leaves the backbone model, latent schedule, and inference-time decoding procedure unchanged. We further formulate this collapse as a geometric bottleneck in implicit reasoning, casting its mitigation as a training-time regularization problem rather than an inference-time decoding change. Extensive experiments show that \shortmethod{} improves matched implicit CoT baselines, while mechanistic analyses reveal higher effective rank, lower dominant-direction energy, and reduced redundancy among latent steps. These results highlight latent subspace utilization as an important factor for efficient continuous reasoning, providing a geometric perspective for analyzing and improving implicit CoT. Code is available at \url{https://github.com/whitesweater/WIDER}.
\end{abstract}

%% file: sections/sec1_introduction.tex
\section{Introduction}

Chain-of-thought (CoT) prompting is a powerful but costly interface for eliciting reasoning in large language models (LLMs). Prior work shows that explicit intermediate steps, from few-shot and zero-shot CoT to self-consistency and least-to-most reasoning, improve arithmetic, commonsense, symbolic, and compositional reasoning \cite{wei2022chain,kojima2022large,zhou2022least,wang2022self,merrill2024expressive}. Yet explicit CoT externalizes computation as natural-language tokens, increasing generated tokens, latency, context usage, and deployment cost \cite{he2026semcot, wang2025synadapt}, while constraining reasoning to the vocabulary space and mixing computation with communicative or textual-coherence tokens \cite{hao2024training,xu2025softcot,shen2025codi,wei2025sim}. These limitations have motivated implicit reasoning, where intermediate computation is kept silent or latent rather than fully verbalized \cite{li2025implicit}.

Implicit CoT addresses this tradeoff by preserving useful intermediate computation while avoiding full rationale decoding.
  Recent work in this direction can be roughly grouped into three categories according to how the latent computation is introduced.
  First, pause, filler, and thought-control token methods provide additional internal computation budget without exposing all intermediate reasoning as final text \cite{goyal2024think,zelikman2024quiet}.
  Second, continuous-thought and compressed-thought methods allow hidden states to directly participate in subsequent reasoning \cite{hao2024training,cheng2024compressed,shen2025codi}.
  Third, distillation, assistant-model, and auxiliary-supervision methods compress the capability of explicit CoT into compact latent trajectories \cite{xu2025softcot,xu2025softcot++,wang2025synadapt,he2026semcot,wei2025sim}.
	  
  Despite this progress, implicit CoT still has a reliability gap: useful latent computation is not necessarily organized into a stable trajectory.
  Although prior studies show that implicit or continuous reasoning can match or surpass explicit CoT in some settings, latent states may also suffer from training instability, semantic homogenization, and diversity collapse, becoming too similar to encode distinct reasoning steps \cite{shen2025codi,xu2025softcot,wang2025synadapt,he2026semcot,wei2025sim}.
  \textbf{This motivates the central question: do multiple latent steps span distinct computational directions, or do they collapse toward a shared mode?}

Representation collapse provides a concise lens for this instability: latent states can concentrate in a narrow subspace instead of using distinct computational directions.
  In Transformers, low-rank behavior has been linked to token uniformity, degraded signal propagation, and vanishing query/key gradients \cite{dong2021attention,noci2022signal,nait2025mind}.
  In representation learning, related failures are commonly diagnosed as redundancy, reduced effective rank, covariance collapse, or dimensional collapse \cite{jing2021understanding,zbontar2021barlow,bardes2021vicreg,chen2026rankup}.
  For implicit CoT, this suggests that adding latent steps does not guarantee added computation; multiple states may instead place excessive energy on the same dominant direction. We refer to this failure mode as \emph{latent rank collapse}.

  \textbf{To mitigate latent rank collapse, we introduce \shortmethod{}, a lightweight spectral regularizer that improves the geometry of implicit CoT during training while leaving inference unchanged.}
  First, it penalizes excessive alignment with the shared dominant direction of each latent trajectory, encouraging latent states to occupy a broader representational subspace.
  Second, it is plug-and-play: it adds no inference-time modules, tokens, or decoding steps, and keeps the backbone, latent schedule, and decoder interface unchanged.
  In matched comparisons, this regularization improves reasoning accuracy and produces latent trajectories with higher effective rank, lower dominant-direction energy, and reduced redundancy, suggesting that the gain comes from better latent organization rather than surface output shifts.

Our contributions are:
\begin{itemize}[leftmargin=1.4em,itemsep=2pt,topsep=2pt]
    \item We formulate \emph{latent rank collapse} as a geometric failure mode of implicit CoT, where multiple latent states overuse a shared dominant direction instead of providing distinct computational directions.
    \item We propose \shortmethod{}, a plug-in spectral regularizer that suppresses this shared direction during training, encourages a broader latent subspace, and leaves the inference procedure unchanged.
    \item Extensive results show that \shortmethod{} improves matched implicit CoT baselines across multiple backbones and reasoning benchmarks, while focused ASDiv-Aug analyses connect the gains to higher effective rank, lower dominant-direction energy, and reduced latent-step redundancy.
\end{itemize}

%% file: sections/sec2_related_work.tex
\section{Related Work}
\label{sec:related-work}

\paragraph{Explicit chain-of-thought reasoning.}
Chain-of-thought (CoT) reasoning shows that large language models benefit from generating intermediate reasoning steps before final answers \cite{wei2022chain}. This idea has been extended by zero-shot reasoning prompts \cite{kojima2022large}, answer aggregation over multiple sampled paths \cite{wang2022self}, and problem decomposition \cite{zhou2022least}. Other work improves or restructures explicit reasoning through rationale bootstrapping \cite{zelikman2022star}, program-aided reasoning \cite{chen2022program}, reasoning-action interaction \cite{yao2022react}, and tree- or graph-structured search over thoughts \cite{yao2023tree,besta2024graph}. Although these methods differ in form, they share the same principle: intermediate computation is useful for complex reasoning. Their main limitation is that such computation is usually externalized as natural-language tokens, which increases latency and context usage while forcing internal states to be verbalized. This motivates latent reasoning methods that keep part of the reasoning process in continuous space.
  
% \paragraph{Implicit and latent reasoning.} 
% Chain-of-thought (CoT) reasoning improves language-model reasoning by generating intermediate steps before final answers \cite{wei2022chain}. This idea has been extended through zero-shot prompting and self-consistency \cite{kojima2022large,wang2022self}, problem decomposition and rationale bootstrapping \cite{zhou2022least,zelikman2022star}, and structured reasoning with programs, actions, trees, or graphs \cite{chen2022program,yao2022react,yao2023tree,besta2024graph}. These methods show the value of intermediate computation, but usually express it as decoded natural-language tokens, which increases latency and context usage and forces internal computation to be verbalized. This motivates latent reasoning methods that keep part of the reasoning process in continuous space.

\textbf{Implicit and latent reasoning.}
Recent work moves intermediate computation into continuous hidden states, avoiding the need to verbalize every reasoning step. Coconut feeds continuous thoughts back into the model, while CCoT compresses explicit rationales into compact latent representations \citep{hao2024training,cheng2024compressed}. CODI transfers explicit-CoT capability into latent trajectories through self-distillation, and SoftCoT generates instance-specific soft thoughts with an auxiliary model \citep{shen2025codi,xu2025softcot}. Other methods use step-level supervision, semantic alignment, or adaptive latent reasoning to improve stability and efficiency \citep{wei2025sim,he2026semcot,wang2025synadapt}. These works mainly study how latent reasoning is constructed or supervised; we instead examine the geometry of the resulting latent trajectory.

\paragraph{Representation collapse and spectral regularization.} 
Representation collapse provides a natural lens for this question. Transformer and neural representations can concentrate in low-dimensional subspaces \cite{dong2021attention,noci2022signal}, a phenomenon also studied as dimensional collapse in self-supervised learning \cite{jing2021understanding}. Anti-collapse methods reduce redundancy, preserve variance, or rebalance spectra \cite{zbontar2021barlow,bardes2021vicreg,zhang2023spectral}, while related studies connect collapse to reasoning stability, spectral gaps, and architectural rank limits \cite{arefin2025seq,saada2024mind,chen2026rankup}. We therefore regularize the singular-value structure of the per-example latent-token matrix, encouraging implicit reasoning states to use a broader subspace without changing the architecture or inference procedure.

%% file: sections/sec3_rank_collapse.tex
\section{Understanding Latent Rank Collapse}
\label{sec:latent-rank-collapse}

We formalize latent rank collapse as a per-example geometric failure mode of
implicit CoT: the model allocates multiple latent steps, but their hidden states
occupy nearly the same direction. For an input $\mathbf{x}$, write the latent
trajectory as
\begin{equation}
    \mathbf{H}(\mathbf{x})=[\mathbf{h}_1(\mathbf{x});\ldots;\mathbf{h}_m(\mathbf{x})]\in\mathbb{R}^{m\times d},
\end{equation}
where each row $\mathbf{h}_i(\mathbf{x})\in\mathbb{R}^d$ is the hidden state read out at the
$i$-th implicit reasoning position before the answer token, following prior work
on latent reasoning \citep{shen2025codi}. Here $m$ is the number of latent
positions and $d$ is the model hidden size. To focus on directional diversity
rather than hidden-state norms, we also use the row-normalized trajectory
\begin{equation}
    \begin{aligned}
    \mathbf{z}_i(\mathbf{x})
    &=\frac{\mathbf{h}_i(\mathbf{x})}{\|\mathbf{h}_i(\mathbf{x})\|_2+\varepsilon}, \\
    \mathbf{Z}(\mathbf{x})
    &=[\mathbf{z}_1(\mathbf{x});\ldots;\mathbf{z}_m(\mathbf{x})].
    \end{aligned}
    \label{eq:rank-collapse-z}
\end{equation}
The central geometric question is whether multiple allocated latent positions occupy a broad set of directions or mostly repeat a shared representation. 
Prior work reports related symptoms in implicit or continuous reasoning, including semantic homogenization across latent steps and reduced latent diversity~\citep{wei2025sim}.

\definitionbox{%
\textbf{Definition.}
We use \emph{latent rank collapse} to describe the failure mode in which the normalized rows in $\mathbf{Z}(\mathbf{x})$ concentrate around a dominant direction $\mathbf{v}_1$. Equivalently, the unnormalized states can often be
viewed as $\mathbf{h}_i(\mathbf{x})\!\approx\!\alpha_i(\mathbf{x})\mathbf{v}_1$,
so latent steps differ mostly by scale and add little new directional
variation. The trajectory need not be exactly rank one; the key issue is that
successive steps use only a small part of the available latent subspace.}
Figure~\ref{fig:rank_collapse_concept} illustrates this geometry: useful latent
computation spreads across directions, whereas collapse makes additional steps
align with a shared mode.

\begin{figure}[t]
    \centering
    \includegraphics[width=\columnwidth]{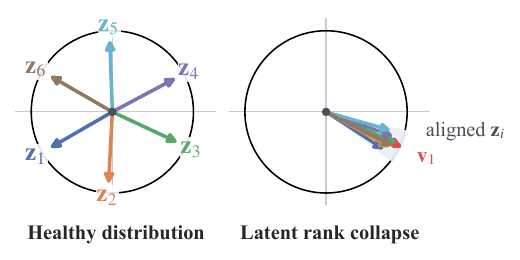}
    \caption{Conceptual illustration of latent rank collapse. Each arrow is a
    unit-length latent state. In the healthy distribution (\textbf{left}), states span many
    directions. In latent rank collapse (\textbf{right}), states concentrate near one
    dominant direction, so hidden states mainly differ by scalar coefficients.}
    \label{fig:rank_collapse_concept}
\end{figure}

\paragraph{Why rank matters.}
$\mathbf{H}(\mathbf{x})$ is the channel through which the decoder receives input-specific computation from the implicit reasoning stage. If all latent steps vary mainly along one direction, the decoder receives more positions but not many independent directions of computation. The conditional family $\{p(y\mid \mathbf{x})\}_{\mathbf{x}}$ is then bottlenecked by the low-rank coordinates carried by the latent trajectory, even when the number of latent steps $m$ is large. Raising the effective rank of $\mathbf{H}(\mathbf{x})$ can
therefore enlarge per-input conditional capacity while leaving the marginal
$p(y)$ unchanged.

\begin{proposition}[Subspace factorization]
\label{prop:capacity-bound}
Let the decoder factor through the latent trajectory as
$p(y\mid \mathbf{x})=g(\mathbf{H}(\mathbf{x}))$. If, for every
$\mathbf{x}\in\mathcal{X}$, the rows of $\mathbf{H}(\mathbf{x})$ lie in a common
$r$-dimensional subspace with basis $\mathbf{U}\in\mathbb{R}^{d\times r}$, so
that $\mathbf{H}(\mathbf{x})=\mathbf{A}(\mathbf{x})\mathbf{U}^\top$ with
$\mathbf{A}(\mathbf{x})\in\mathbb{R}^{m\times r}$, then there exists
$\tilde g:\mathbb{R}^{m\times r}\to\Delta(\mathcal{Y})$ such that
$p(y\mid \mathbf{x})=\tilde g(\mathbf{A}(\mathbf{x}))$ for all $\mathbf{x}$.
In particular, rank-one collapse reduces the family to a function of $m$ scalar
coefficients along one direction.
\end{proposition}

The proof and detailed scope qualifications are given in Appendix~\ref{app:capacity-bound}. Related links between rank and expressiveness also appear in Transformer attention and self-supervised representation learning \citep{merrill2024expressive,dong2021attention,bhojanapalli2020low,nait2025mind,jing2021understanding,bardes2021vicreg}.
Our experiments report shared-direction dominance and step-to-step redundancy
as geometric diagnostics of within-trajectory concentration.

\paragraph{Why collapse can emerge.}
A compact local model provides one possible way in which this geometric pattern can arise. In implicit CoT,
the model repeatedly updates continuous states before decoding the answer
\citep{shen2025codi}. If the local update around a trajectory amplifies one
direction more than others, repeated applications can make successive latent
states increasingly aligned, as in power iteration. We do not assume the model
is globally linear; the point is to motivate measuring the dominant shared
direction and penalizing it during training. Appendix~\ref{app:local-mode} gives
the derivation, and Figure~\ref{fig:rank_collapse} shows the same pattern
empirically: the baseline has nearly identical latent steps with effective rank
$1.49$, while the regularizer restores visible diversity and raises effective
rank to $3.82$.

% \paragraph{Design goal.}
% The regularizer should be per-example, cheap to compute, compatible with the original
% training loss, and absent at inference time. It should not force the model to spread latent
% states arbitrarily. Instead, it should discourage the most direct failure mode: multiple
% latent steps placing excessive energy on the same shared direction.

%% file: sections/sec4_method.tex
\section{Method}
\label{sec:method}

We propose \shortmethod{}, a training-time regularizer for latent rank
collapse in implicit reasoning. For each example, it estimates the dominant
shared direction of the latent trajectory and penalizes projections onto that
direction.

The method follows the geometry in Section~\ref{sec:latent-rank-collapse}: if
successive states overuse one direction, extra latent steps add little new
computation. \shortmethod{} changes only the training objective; the backbone,
latent schedule, and decoding procedure are unchanged.

\subsection{Latent states and row normalization}
\label{subsec:row-normalization}

For each input $\mathbf{x}$, we use the row-normalized trajectory
$\mathbf{Z}(\mathbf{x})$ defined in Eq.~\eqref{eq:rank-collapse-z}. This makes
the auxiliary loss depend on the direction of each latent state rather than on
its absolute norm. The base training loss remains responsible for learning the
answer; the auxiliary loss only controls redundancy among latent steps.

\subsection{Shared-direction proxy}
\label{subsec:shared-direction-proxy}

We estimate the shared direction of the latent trajectory by summing the normalized
states:
\begin{equation}
    \mathbf{c}(\mathbf{x})=\mathbf{Z}(\mathbf{x})^\top\mathbf{1}_m=\sum_{i=1}^{m} \mathbf{z}_i(\mathbf{x}).
\end{equation}
The unit proxy is then
\begin{equation}
    \mathbf{q}(\mathbf{x})=\sg\left(\frac{\mathbf{c}(\mathbf{x})}{\|\mathbf{c}(\mathbf{x})\|_2+\varepsilon}\right),
    \label{eq:q-proxy}
\end{equation}
where $\sg(\cdot)$ denotes stop-gradient. Intuitively, if many latent states point in a
similar direction, their sum identifies that shared component. If the states are already
well diversified, the summed direction is weaker and less dominant.

The proxy in Eq.~\eqref{eq:q-proxy} is not introduced as an exact top singular-vector
estimator. Its role is operational: it is a cheap direction that summarizes what the current
latent steps have in common. Appendix~\ref{app:mean-coherence} relates this column-sum
direction to average pairwise alignment among latent states, and
Appendix~\ref{app:spectral-connection} gives the corresponding spectral interpretation.
Appendix~\ref{app:spiked-direction} further shows that under a spiked latent-step model,
the mean direction and the leading right singular direction identify the same shared
component, which is why this lightweight proxy can replace an SVD inside the training loop.

\subsection{Projection regularization}
\label{subsec:projection-regularization}

The auxiliary loss penalizes the average squared projection onto the shared direction:
\begin{equation}
\begin{aligned}
    \mathcal{L}_{\mathrm{spec}}(\mathbf{x})
    &= \frac{1}{m}\sum_{i=1}^{m}\bigl(\mathbf{z}_i(\mathbf{x})^\top \mathbf{q}(\mathbf{x})\bigr)^2 \\
    &= \frac{1}{m}\|\mathbf{Z}(\mathbf{x})\mathbf{q}(\mathbf{x})\|_2^2 .
\end{aligned}
    \label{eq:spectral-loss}
\end{equation}
The final training objective is
\begin{equation}
    \mathcal{L}
    =
    \mathcal{L}_{\mathrm{base}}
    +
    \lambda_{\mathrm{spec}}
    \mathbb{E}_{\mathbf{x}\sim\mathcal{D}}\left[
    \mathcal{L}_{\mathrm{spec}}(\mathbf{x})
    \right].
    \label{eq:full-training-objective}
\end{equation}
Here $\mathcal{L}_{\mathrm{base}}$ is the original latent-reasoning training objective. The
spectral term does not replace the task loss and does not introduce additional inference
steps. It only discourages the model from using multiple latent positions to encode the
same shared component.

For batch size $B$, latent-step count $m$, and hidden dimension $d$, the added training
cost is $O(Bmd)$. At inference time, the auxiliary loss is removed, so the latent schedule
and decoding procedure are unchanged.

%% file: sections/sec5_experiment.tex
\section{Experiments}
\label{sec:experiments}

\input{tables/main_results}

We evaluate \shortmethod{} with two implicit-reasoning methods, CODI and SIM-CoT, on two language-model backbones and four reasoning benchmarks.
The empirical study is organized to answer two questions in turn: (i) does adding the spectral loss to a matched implicit baseline give a measurable accuracy gain (\S\ref{sec:main_results})? and (ii) is that gain robust to the regularization strength (\S\ref{sec:ablation})? Section~\ref{sec:mechanism} examines how latent geometry changes alongside the accuracy differences.

\subsection{Experimental Setup}

\paragraph{Baselines.}
We use four reasoning configurations as our baselines. Since \shortmethod{} is a plug-in training objective, we apply it to the two implicit-reasoning methods and compare matched configurations that share the same training data and evaluation parser; the only intended difference within each matched pair is whether the \shortmethod{} loss is added.
\textbf{no-CoT} \cite{wei2022chain} is supervised to produce the final answer directly, without generating intermediate reasoning steps.
\textbf{SFT-CoT} \cite{wei2022chain,kojima2022large} is supervised on CoT-annotated data, where the model generates explicit intermediate reasoning steps before the final answer.
\textbf{CODI} \cite{shen2025codi} is an implicit CoT baseline that compresses explicit reasoning into continuous thought states through self-distillation, transferring reasoning ability by aligning hidden activations at a designated answer-side token.
\textbf{SIM-CoT} \cite{shen2025codi,wei2025sim} further adds an auxiliary decoder during training to align each implicit token with its corresponding explicit reasoning step, providing step-level latent supervision while removing the decoder at inference.
The no-CoT and SFT-CoT settings are reported as contextual references. Because they belong to the explicit-output reasoning paradigm, they are not treated as ablations of the \shortmethod{} loss. The causal comparisons are always made between a matched CODI or SIM-CoT baseline and its \shortmethod{}-regularized counterpart.

\paragraph{Backbones.}
Our main experiments use two open-source LLMs.
\textbf{LLaMA-3.2-3B-Instruct} \cite{grattafiori2024llama} serves as the default backbone for the main matched comparisons, the ablation, and the mechanism analysis, while \textbf{Qwen3-1.7B} \cite{yang2025qwen3} probes whether the spectral signal transfers across backbone families and parameter scales. We additionally evaluate \textbf{LLaMA-3.1-8B} \cite{grattafiori2024llama} in the scaling experiment reported in Appendix~\ref{app:llama31-8b}.
Training and adaptation configurations for the main experiments are provided in Appendix~\ref{app:backbone-config}.

\paragraph{Datasets.}
We evaluate on four reasoning benchmarks spanning three task formats, with a shared prompt template and evaluation parser per backbone.
\textbf{StrategyQA} \cite{geva2021did} is a binary yes/no commonsense benchmark.
\textbf{CommonsenseQA} \cite{talmor2019commonsenseqa} is a five-way multiple-choice commonsense benchmark from ConceptNet.
\textbf{ASDiv-Aug} \cite{miao2020diverse} is an arithmetic word-problem benchmark augmented along the standard CODI/Coconut protocol; its clean numerical answer space makes it our focused mechanism-analysis setting.
\textbf{AQuA} \cite{ling2017program} is a five-way multiple-choice algebraic word-problem benchmark.
Together, these benchmarks cover binary, multiple-choice, and open-numeric answer formats, allowing us to examine whether the effect of \shortmethod{} varies with task format in addition to backbone and decoder supervision.

\subsection{Main Results}
\label{sec:main_results}

Table~\ref{tab:main_results} reports the calibrated result matrix across backbones, implicit-reasoning variants, and benchmarks. The no-CoT and SFT-CoT rows are included as reference points for explicit-output reasoning; the causal comparisons are the matched CODI/SIM-CoT pairs with and without \shortmethod{}. We highlight three observations.

\paragraph{Matched averages improve in every implicit setting.}
Across the four backbone--method combinations (LLaMA-3B/Qwen3-1.7B $\times$ CODI/SIM-CoT), adding \shortmethod{} raises average accuracy by $+3.67\%$, $+1.83\%$, $+3.52\%$, and $+4.40\%$ relative to the corresponding matched baseline. This suggests that the signal is not tied to a single backbone family or to one particular implicit-reasoning interface.

\paragraph{The gains are broad.}
At the task-cell level, fifteen of the sixteen matched comparisons improve. The single exception is small: LLaMA-3B SIM-CoT on CommonsenseQA decreases by $0.49$ accuracy points. We therefore treat \shortmethod{} as a broadly useful plug-in objective.

\paragraph{The largest point gains occur on the arithmetic benchmarks. }
The biggest single-cell gain is Qwen SIM-CoT on ASDiv-Aug ($72.35\!\to\!79.19$, $+6.84$ points), followed by Qwen CODI on AQuA ($37.01\!\to\!42.13$, $+5.12$ points) and LLaMA CODI on AQuA ($28.74\!\to\!33.46$, $+4.72$ points). 
% We report this as a descriptive pattern; the association between the implicit--explicit gap and \shortmethod{} gain is inference-sensitive and is not treated as a general law.

\paragraph{Scaling to 8B.}
We additionally evaluate WIDER on LLaMA-3.1-8B with CODI.
WIDER improves accuracy on all four benchmarks
, with gains of $+2.57$, $+0.32$, $+4.05$, and $+1.18$ percentage points on StrategyQA, CommonsenseQA, ASDiv-Aug, and AQuA, respectively.
Effective rank also increases consistently, from $1.715$--$2.092$
under CODI to $3.999$--$5.523$ with WIDER.
Full results are reported in Appendix~\ref{app:llama31-8b}.

\subsection{Ablation Experiments}
\label{sec:ablation}

We sweep the spectral coefficient $\lambda$ on LLaMA-3B + CODI on StrategyQA to verify that the gain is not an artifact of a single hyperparameter choice and to locate the operating regime. Figure~\ref{fig:ablation_lambda} reports the best calibrated accuracy at each value.

\begin{figure}[t]
\centering
\includegraphics[width=\columnwidth]{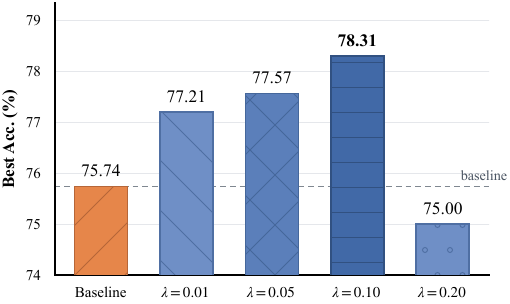}
\caption{Sensitivity to the regularization coefficient $\lambda$ on LLaMA-3B + CODI, StrategyQA. The dashed line marks the baseline ($\lambda=0$); the peak at $\lambda=0.10$ is highlighted.}
\label{fig:ablation_lambda}
\end{figure}

Accuracy improves monotonically as $\lambda$ grows from $0.01$ to $0.10$ and then drops at $\lambda = 0.20$, where the spectral penalty starts to over-regularize the latent trajectory. The peak at $\lambda = 0.10$ is the value used for the matched comparisons in Table~\ref{tab:main_results}; the curve is smooth on both sides of the peak, indicating that the gain is robust to moderate perturbations of $\lambda$ rather than tuned to a knife-edge value.

%% file: tables/main_results.tex
\begin{table*}[!t]
\centering
\normalsize
\renewcommand{\arraystretch}{1.7}
\setlength{\tabcolsep}{6pt}
\resizebox{\textwidth}{!}{
\begin{tabular}{ll cc cc cc cc c}
\toprule
& & \multicolumn{2}{c}{\textbf{StrategyQA}} & \multicolumn{2}{c}{\textbf{CommonsenseQA}} & \multicolumn{2}{c}{\textbf{ASDiv-Aug}} & \multicolumn{2}{c}{\textbf{AQuA}} & \textbf{Avg.} \\
\cmidrule(lr){3-4}\cmidrule(lr){5-6}\cmidrule(lr){7-8}\cmidrule(lr){9-10}
\textbf{Backbone} & \textbf{Method} & Acc & tok & Acc & tok & Acc & tok & Acc & tok & Acc \\
\midrule
\rowcolor{refRowGray} \cellcolor{white}
\multirow{6}{*}[2ex]{\shortstack[c]{LLaMA-\\[0.15ex]3B}}
& no-CoT                  & 68.01           & 3.00   & 79.12             & 3.00  & 88.73             & 4.07  & 49.61              & 3.00  & 71.37 \\
\rowcolor{refRowGray} \cellcolor{white}
& SFT-CoT                 & 78.68           & 125.10 & 77.23             & 80.61 & 93.06             & 12.81 & 47.66              & 75.58 & 74.16 \\
\cmidrule(l){2-11}
& CODI                    & 75.74           & 6.00   & 78.30             & 6.00  & 80.73             & 7.06  & 28.74              & 6.00  & 65.88 \\
\rowcolor{methodRowBlue} \cellcolor{white}
& CODI + \shortmethod{}   & 78.31~\up{3.39\%} & 6.00 & 78.95~\up{0.83\%} & 6.00  & 82.47~\up{2.16\%} & 7.07  & 33.46~\up{16.42\%} & 6.00  & 68.30~\up{3.67\%} \\
& SIM-CoT                 & 73.53           & 6.00   & 79.61             & 6.00  & 80.06             & 7.06  & 27.17              & 5.00  & 65.09 \\
\rowcolor{methodRowBlue} \cellcolor{white}
& SIM-CoT + \shortmethod{} & 75.00~\up{2.00\%} & 6.00 & 79.12~\dn{0.62\%} & 6.04  & 82.27~\up{2.76\%} & 7.06  & 28.74~\up{5.78\%}  & 6.00  & 66.28~\up{1.83\%} \\
\midrule
\midrule
\rowcolor{refRowGray} \cellcolor{white}
\multirow{6}{*}[2ex]{\shortstack[c]{Qwen3-\\[0.15ex]1.7B}}
& no-CoT                  & 61.40           & 3.00   & 73.63             & 3.00  & 85.07             & 5.08  & 44.92              & 3.00  & 66.26 \\
\rowcolor{refRowGray} \cellcolor{white}
& SFT-CoT                 & 69.85           & 126.08 & 72.40             & 79.45 & 91.71             & 17.45 & 47.27              & 78.42 & 70.31 \\
\cmidrule(l){2-11}
& CODI                    & 66.54           & 6.00   & 67.98             & 6.00  & 77.55             & 8.16  & 37.01              & 6.00  & 62.27 \\
\rowcolor{methodRowBlue} \cellcolor{white}
& CODI + \shortmethod{}   & 68.38~\up{2.77\%} & 6.00 & 69.37~\up{2.04\%} & 6.00  & 77.94~\up{0.50\%} & 8.16  & 42.13~\up{13.83\%} & 9.41  & 64.46~\up{3.52\%} \\
& SIM-CoT                 & 63.97           & 6.00   & 67.90             & 6.01  & 72.35             & 8.15  & 41.73              & 6.00  & 61.49 \\
\rowcolor{methodRowBlue} \cellcolor{white}
& SIM-CoT + \shortmethod{} & 64.71~\up{1.16\%} & 6.00 & 68.39~\up{0.72\%} & 6.00  & 79.19~\up{9.45\%} & 8.11  & 44.49~\up{6.61\%} & 6.00  & 64.20~\up{4.40\%} \\
\bottomrule
\end{tabular}
}
\caption{Main results across backbones, methods, and reasoning benchmarks. Each task is shown as two side-by-side sub-columns: \textit{Acc} (accuracy in \%) and \textit{tok} (the per-cell average generation length). \textit{Avg.} is the unweighted mean of the four task accuracies. The first two rows of each backbone are explicit-reasoning references (no-CoT and SFT-CoT). Parentheses on the four \shortmethod{}-regularized rows give the relative improvement of the \shortmethod{} variant over its matched CODI / SIM-CoT baseline in the same backbone (computed as $(v_{\mathrm{spec}}-v_{\mathrm{base}})/v_{\mathrm{base}}$).}
\label{tab:main_results}
\end{table*}

%% file: sections/sec6_discussion.tex
\section{Mechanistic Analysis}
\label{sec:mechanism}

The main results show that \shortmethod{} improves the matched baselines;
this section asks \emph{why}. Unless otherwise specified, the analysis in
this section uses ASDiv-Aug with LLaMA-3B + SIM-CoT.
We organize the analysis around four candidate explanations --- token-level
repetition, dominant-direction shortcut, shortcut content, and a shifted
answer prior --- and ask which of them, if any, the spectral loss actually
moves.

\subsection{Are Latent Tokens Repetitive?}
\label{sec:latent_mi}

The most direct failure mode is token-level repetition: later latent steps
reproduce earlier states. We measure pairwise dependence between the six
latent steps with normalized mutual information (NMI). Latent vectors are
standardized, PCA-projected, and discretized with a shared
MiniBatchKMeans codebook of size $k$; NMI is then computed between latent-step
cluster assignments. Higher off-diagonal NMI means stronger repetition.

\begin{figure}[t]
    \centering
    \includegraphics[width=\columnwidth]{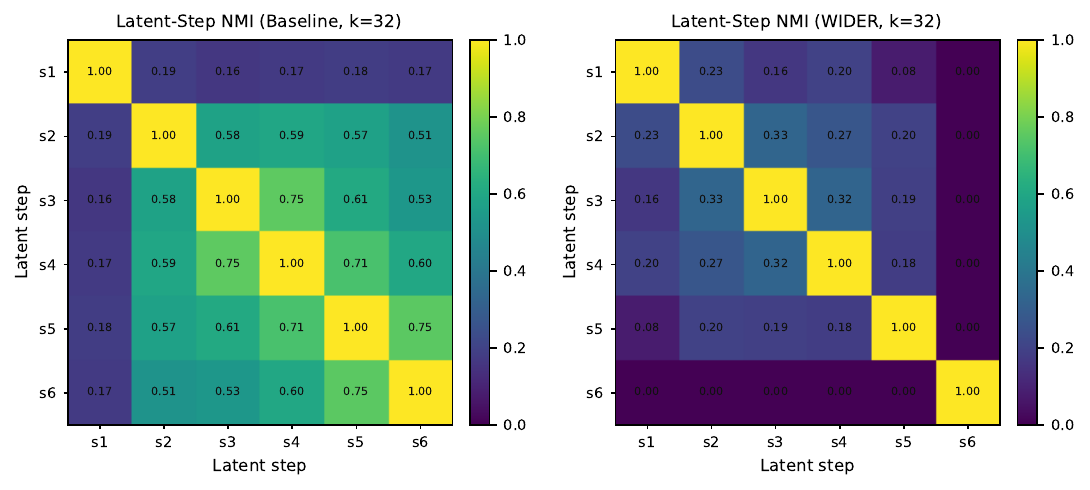}
    \caption{Latent-step NMI heatmaps on ASDiv-Aug (SIM-CoT, $k{=}32$). The
    diagonal is fixed at $1.0$; signal lies in the off-diagonals. The baseline
    (left) shows a dense high-dependence block among later steps; \shortmethod{}
    (right) sharply reduces this block, leaving $s_6$ nearly independent of
    earlier states.}
    \label{fig:asdiv_nmi_heatmap}
\end{figure}

Figure~\ref{fig:asdiv_nmi_heatmap} shows the qualitative shift, and
Table~\ref{tab:asdiv_nmi_mean} the averages: mean pairwise NMI drops from
$0.472$ to $0.144$ at $k{=}32$, and from $0.306$ to $0.142$ at $k{=}16$. The
baseline therefore spends part of its latent-token budget revisiting earlier
states, and \shortmethod{} reclaims that budget by making the steps carry
distinguishable cluster assignments.
A natural next question is whether those now-distinguishable steps
actually span independent directions of the latent space, or merely
re-weight a single shared one.

\begin{table}[t]
\centering
\small
\setlength{\tabcolsep}{5pt}
\begin{tabular}{lcccc}
\toprule
Codebook & Baseline & +\shortmethod{} & $\Delta$ & Ratio \\
\midrule
$k{=}16$ & $0.306$ & \textcolor{gainGreen}{$0.142$} & \textcolor{lossRed}{$-0.164$} & $0.46$ \\
$k{=}32$ & $0.472$ & \textcolor{gainGreen}{$0.144$} & \textcolor{lossRed}{$-0.328$} & $0.31$ \\
\bottomrule
\end{tabular}
\caption{Mean pairwise latent-step NMI on ASDiv-Aug (SIM-CoT). \shortmethod{} roughly
halves NMI at $k{=}16$ and cuts it by $\sim$70\% at $k{=}32$.}
\label{tab:asdiv_nmi_mean}
\end{table}

\subsection{Does \shortmethod{} Release the Trajectory from a Dominant Shortcut?}
\label{sec:dominant_tail}

NMI tests pairwise step repetition; a complementary failure is that all six
steps, while pairwise distinct, still concentrate their energy on a single
shortcut direction. We track this with
$\mathrm{DominantEnergy}=\sigma_1^2/\sum_i\sigma_i^2$ and
$\mathrm{TailEnergy}=1-\mathrm{DominantEnergy}$ over the latent trajectory.

\begin{figure}[t]
    \centering
    \includegraphics[width=1\columnwidth]{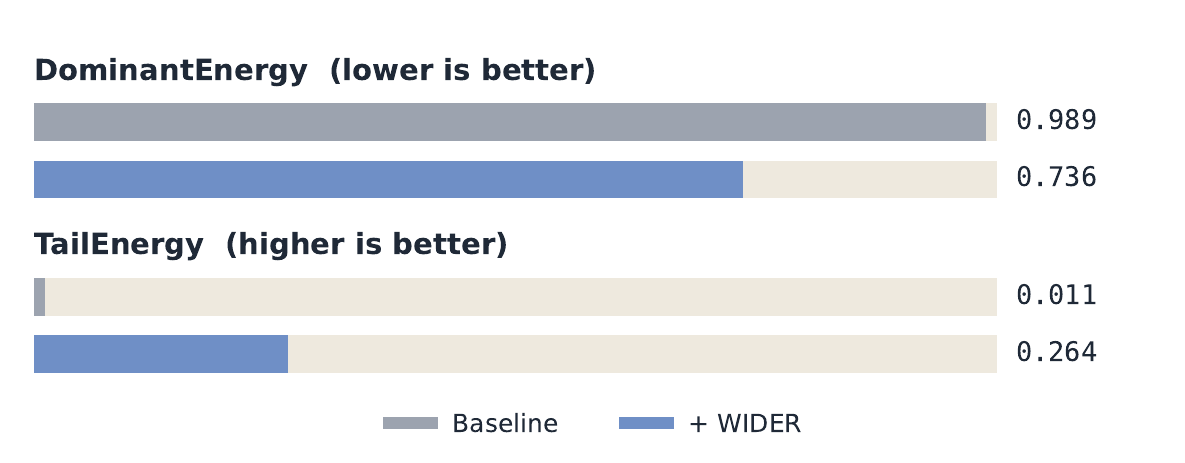}
    \caption{Dominant and tail energy on ASDiv-Aug (SIM-CoT, LLaMA-3B).
    \shortmethod{} redistributes energy from the first singular direction into the
    spectral tail.}
    \label{fig:asdiv_energy}
\end{figure}

Figure~\ref{fig:asdiv_energy} shows the baseline is effectively rank-one: the
top direction carries $0.989$ of the energy, leaving only $0.011$ in the tail.
\shortmethod{} moves the split to $0.736$/$0.264$ --- the latent trajectory now uses
non-dominant directions in earnest rather than routing all computation through
one channel.

\paragraph{Where the gain accrues across latent steps.}
The same redistribution shows up in effective rank, which rises from $1.487$
to $3.817$. To localize \emph{which} steps gain, we compute prefix effective
rank --- the effective rank of the first $t$ states as $t$ grows from $1$ to
$6$. Figure~\ref{fig:prefix_rank} shows the baseline saturates near $1.5$
after the first few steps, while \shortmethod{} keeps accumulating rank through
$s_6$, so later steps contribute new directions instead of re-tracing the
early trajectory. The fact that the rank gain accrues at later steps,
rather than being a uniform offset, is consistent with collapse being a
step-wise drift onto the dominant direction that \shortmethod{} resists
from $s_2$ onward.

\begin{figure}[t]
    \centering
    \includegraphics[width=0.75\columnwidth]{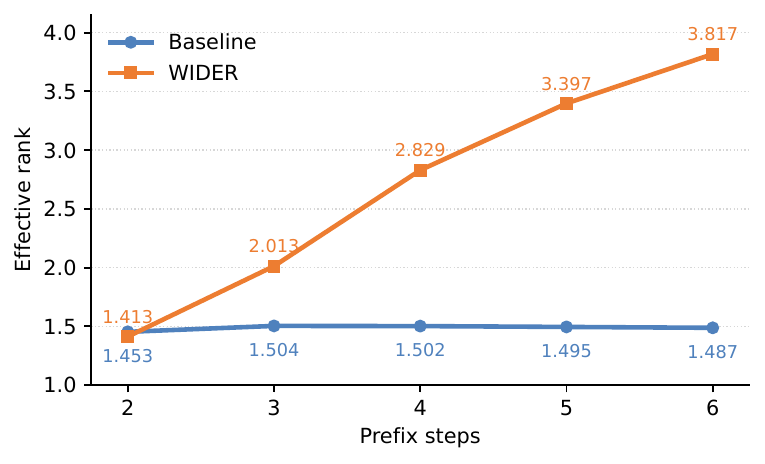}
    \caption{Prefix effective rank on ASDiv-Aug (SIM-CoT), over the first six
    latent states.}
    \label{fig:prefix_rank}
\end{figure}

\subsection{What Is Encoded in the Dominant Direction?}
\label{sec:shortcut_content}

Reducing dominant-direction concentration is only meaningful if that direction
carries low-value signal in the first place. To probe its content, we project
the leading right singular vector of each latent trajectory through the
language-model head, take the top vocabulary items (using absolute scores,
since singular-vector sign is arbitrary), and split them into two buckets:
\emph{content-like} tokens and everything else (punctuation, function words,
numeric literals, and whitespace/special tokens). We report
$\mathrm{ContentShare}$ together with
$\mathrm{ShortcutShare}=1-\mathrm{ContentShare}$.

Table~\ref{tab:asdiv_projection} shows the baseline dominant direction is
$16.03\%$ shortcut-like; after \shortmethod{} this drops to $0.00\%$, with content
share rising correspondingly. So \shortmethod{} does not merely weaken the dominant
direction in magnitude; it also clears away the shortcut-flavored vocabulary
items aligned with it. We treat this as a diagnostic rather than a causal
statement, but it is consistent with the broader picture from
\S\ref{sec:dominant_tail}. The full categorization protocol is given in
Appendix~\ref{app:vocab-projection}.

\begin{table}[t]
\centering
\small
\setlength{\tabcolsep}{5pt}
\begin{tabular}{lcc}
\toprule
Setting & Shortcut\,$\downarrow$ & Content\,$\uparrow$ \\
\midrule
Baseline & $16.03\%$ & $83.97\%$ \\
+\shortmethod{} & \textcolor{gainGreen}{$0.00\%$} & \textcolor{gainGreen}{$100.00\%$} \\
\bottomrule
\end{tabular}
\caption{Dominant-direction vocabulary projection on ASDiv-Aug (SIM-CoT,
LLaMA-3B). Shortcut share is the complement of content share, i.e.\ everything
that is not a content-like token.}
\label{tab:asdiv_projection}
\end{table}

\subsection{Is the Gain a Shifted Answer Prior?}
\label{sec:output_distribution}

A final alternative is that \shortmethod{} simply reshapes the marginal answer
distribution $p(y)$ over the test set --- for instance by flattening a few
high-frequency answers --- and that accuracy moves as a side effect.

Figure~\ref{fig:output_distribution} rules this out: top-1 prediction share,
normalized entropy, unique-prediction rate, and prediction--gold cosine all
move by less than one percentage point in absolute terms and are all smaller
than the accuracy gap. The marginal $p(y)$ is essentially
preserved, while per-example accuracy rises --- so the gain must come from a
more input-discriminative latent trajectory rather than from a redistributed
answer prior.

\paragraph{A three-way consistency argument.}
Marginal $p(y)$ unchanged, per-example accuracy up, and effective rank from
$1.487$ to $3.817$ together admit one reading: the gain is a drop in average
$-\log p(y^\star\mid \mathbf{x})$ enabled by a higher-rank, more input-discriminative
$\mathbf{H}(\mathbf{x})$, not a reshuffled prior. This matches the pattern in
\S\ref{sec:main_results}: the larger gains appear on AQuA and ASDiv-Aug, whose
answer space most stresses the conditional bottleneck, while binary
StrategyQA --- where one bit of input-specific signal often suffices ---
shows smaller deltas.

\input{tables/output_distribution}

%% file: tables/output_distribution.tex
\begin{figure}[t]
\centering
\includegraphics[width=\columnwidth]{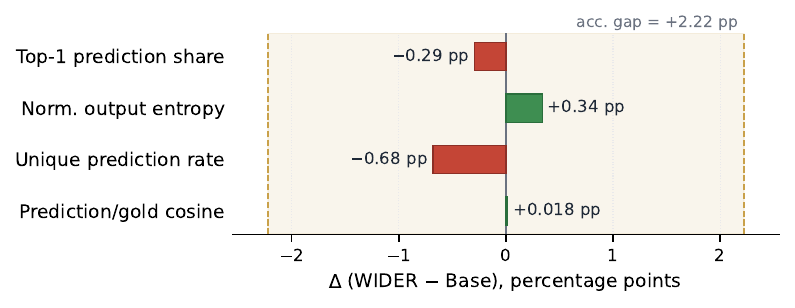}
\caption{Output-distribution shifts on ASDiv-Aug (SIM-CoT). All four
marginal-$p(y)$ metrics move by less than one percentage point and are all
smaller than the $+2.22$\,pp accuracy gap (dashed reference band).}
\label{fig:output_distribution}
\end{figure}

%% file: sections/sec7_conclusion.tex
\section{Conclusion}

Implicit chain-of-thought moves intermediate computation into continuous latent states, but successive states can collapse toward a shared direction, leaving additional latent positions underused. We identified this as \emph{latent rank collapse} --- a per-example geometric failure that bottlenecks $\{p(y\mid\mathbf{x})\}$ by the rank of the latent trajectory --- and proposed \shortmethod{}, a training-time spectral regularizer that penalizes projections onto this shared direction. Across multiple backbones and reasoning benchmarks, \shortmethod{} delivers consistent accuracy gains over matched implicit-CoT baselines. These results highlight latent subspace utilization as a key factor for efficient implicit reasoning, and offer a geometric perspective for analyzing and improving implicit chain-of-thought.

%% file: sections/sec8_limitations.tex
\newpage
\section*{Limitations}

This study has several limitations. First, our experiments evaluate models
up to $8$B parameters across a limited set of reasoning benchmarks. Whether
the observed benefits persist for substantially larger models and
larger-scale training regimes remains to be explored.
  Second, our current experiments use a fixed trajectory of six latent tokens.
Investigating how latent rank collapse and \shortmethod{} behave with longer
or dynamically allocated latent trajectories is an interesting direction for
future work.
Third, because \shortmethod{} is designed as a plug-in objective for latent
reasoning, it inherits the interpretability limitations of latent reasoning
itself. Although our geometric diagnostics characterize the organization of
the latent trajectory, they do not make individual latent tokens directly
human-interpretable or identify the specific reasoning role performed by each
token.

%% file: sections/sec9_ethics.tex
\section*{Ethics Statement}

This paper presents work whose goal is to advance the field of Machine Learning
by addressing \emph{latent rank collapse} in implicit chain-of-thought
reasoning. The proposed spectral regularizer is a training-time mechanism that
encourages latent trajectories to span a broader subspace, and it leaves the
backbone architecture and inference procedure unchanged.

Because the method adds no extra tokens or modules at inference, it does not
increase deployment-time compute, latency, or energy use relative to the
implicit-CoT baseline. However, it targets only the geometry of latent
reasoning and does not by itself address factual correctness, social bias, or
harmful content; models trained with it may still produce incorrect or biased
answers. All datasets used are publicly available reasoning benchmarks and
involve no human subjects or personally identifiable information.

%% file: sections/sec10_acknowledgments.tex
\section*{Acknowledgments}

We thank the anonymous reviewers and area chairs for their constructive feedback.
This work was partially supported by the Guangdong Basic and Applied Basic
Research Foundation (Grant No.~2025A1515140110) and the Research Travel Grant
from the Base of Red Bird MPhil (RBM) at The Hong Kong University of Science and
Technology (Guangzhou). We would like to thank our RBM Project Supervisor Lanhui Li for academic support.

%% file: appendix/appendix.tex
\section{Additional Mathematical Details}
\label{app:math-details}

This appendix collects the short derivations and scope qualifications behind the geometric interpretations in the
main text. Section~\ref{app:capacity-bound} gives the common-subspace factorization
and a local power-iteration illustration used in Section~\ref{sec:latent-rank-collapse};
Section~\ref{app:mean-coherence} justifies the mean-direction proxy and the
projection loss in Section~\ref{sec:method}. None of the arguments assumes that
the full nonlinear model is globally linear.

\subsection{\texorpdfstring{Common-subspace factorization and a local concentration scenario}{Common-subspace factorization and a local concentration scenario}}
\label{app:capacity-bound}
\label{app:local-mode}

\paragraph{Common-subspace factorization.}
Proposition~\ref{prop:capacity-bound} in the main text states an exact
reduced-coordinate factorization under a single fixed subspace shared across inputs. In the rank-one setting
$\mathbf{h}_i(\mathbf{x})=\alpha_i(\mathbf{x})\mathbf{v}_1$, the trajectory
matrix becomes
$\mathbf{H}(\mathbf{x})=\boldsymbol{\alpha}(\mathbf{x})\mathbf{v}_1^\top$ and
is represented relative to the fixed direction $\mathbf{v}_1$ by the $m$ scalar coordinates
$\boldsymbol{\alpha}(\mathbf{x})$, regardless of how large $d$ is.

\begin{proof}[Proof of Proposition~\ref{prop:capacity-bound}]
$\mathbf{U}$ is fixed, so $\mathbf{H}(\mathbf{x})$ is determined by $\mathbf{A}(\mathbf{x})$ via $\mathbf{H}(\mathbf{x})=\mathbf{A}(\mathbf{x})\mathbf{U}^\top$; take
$\tilde g(\mathbf{A})=g(\mathbf{A}\mathbf{U}^\top)$.
\end{proof}

\paragraph{Effect of row normalization.}
With the $\varepsilon$-stabilized normalization in Eq.~\eqref{eq:rank-collapse-z},
the normalized trajectory can be written as
$\mathbf{Z}(\mathbf{x})=\mathbf{D}(\mathbf{x})\mathbf{H}(\mathbf{x})$, where
$\mathbf{D}(\mathbf{x})$ is a positive invertible diagonal matrix. This rescaling
preserves the row span and algebraic rank while reducing the diagnostics'
sensitivity to variation in row norms.

The proposition is an exact reduced-coordinate reparameterization under a
single fixed subspace shared across inputs, not an information-theoretic bound on
real-valued coordinates.  Our trajectory diagnostics (effective rank, dominant-energy ratio, and latent-step
cosine) are computed separately for each example; NMI is instead an across-example
statistic of dependence between latent-step assignments. The concentrated subspace may vary with $\mathbf{x}$. The trajectory metrics
diagnose within-trajectory geometric concentration; they do not directly measure
semantic reasoning capacity. Related
rank--expressiveness arguments appear in the Transformer-expressiveness line
\citep{merrill2024expressive,dong2021attention,bhojanapalli2020low,nait2025mind}
and in the self-supervised representation-collapse line
\citep{jing2021understanding,bardes2021vicreg,zbontar2021barlow}.

\paragraph{Local power-iteration picture.}
For a fixed $\mathbf{x}$, write deviations $\boldsymbol{\delta}_t=\mathbf{h}_t-\mathbf{h}^\star$ around a local reference
state $\mathbf{h}^\star$ and suppose the local update is approximated by
$\boldsymbol{\delta}_{t+1}=\mathbf{A}_{\mathbf{x}}\boldsymbol{\delta}_t$, with $\mathbf{A}_{\mathbf{x}}$ diagonalizable, eigenvalues
$|\lambda_1|>|\lambda_2|\geq\cdots\geq|\lambda_d|$, and $\boldsymbol{\delta}_1$ having a
nonzero component on $\mathbf{v}_1$. Expanding $\boldsymbol{\delta}_1=\sum_k c_k \mathbf{v}_k$ gives
$\boldsymbol{\delta}_t=\sum_k c_k\lambda_k^{t-1}\mathbf{v}_k$, so after normalization
\begin{equation}
    \frac{\boldsymbol{\delta}_t}{\|\boldsymbol{\delta}_t\|_2}\rightarrow
    \pm\frac{\mathbf{v}_1}{\|\mathbf{v}_1\|_2},
\end{equation}
at rate controlled by $|\lambda_2/\lambda_1|$ per step. Under these stated assumptions, repeated application of
the same local update is one possible geometric route by which additional latent
steps occupy few new directions in representation space.

\subsection{Justifying the mean-direction proxy and the projection loss}
\label{app:mean-coherence}
\label{app:spectral-connection}
\label{app:spiked-direction}

Let $\mathbf{Z}\in\mathbb{R}^{m\times d}$ be the row-normalized latent trajectory and
$\bar{\mathbf{z}}=\frac{1}{m}\sum_{i=1}^m \mathbf{z}_i$. Three short arguments show why penalizing
projections onto the normalized mean direction is principled.

\paragraph{Mean direction summarizes pairwise alignment.}
For the $\varepsilon$-stabilized rows,
\begin{equation}
    \|\bar{\mathbf{z}}\|_2^2=\frac{1}{m^2}\sum_{i,j}\mathbf{z}_i^\top \mathbf{z}_j .
\end{equation}
The squared norm of the mean direction is therefore an exact summary of
average pairwise dot products. Because
$\|\mathbf{z}_i\|_2=\|\mathbf{h}_i\|_2/(\|\mathbf{h}_i\|_2+\varepsilon)$,
this is a norm-weighted cosine identity; it coincides with average pairwise
cosine alignment for exactly unit-normalized nonzero rows and approximates it
when $\varepsilon$ is negligible relative to their norms. This is why the column-sum
direction in Eq.~\eqref{eq:q-proxy} is a natural proxy for what the latent steps
share. For any unit $\mathbf{q}$ and $a_i=\mathbf{z}_i^\top \mathbf{q}$, Jensen's inequality gives
\begin{equation}
    \frac{1}{m}\sum_i a_i^2\geq\left(\frac{1}{m}\sum_i a_i\right)^2,
\end{equation}
which equals $\|\bar{\mathbf{z}}\|_2^2$ when $\mathbf{q}$ is the normalized mean. The projection
loss therefore upper-bounds the squared strength of the shared mean direction.

\paragraph{Spectral interpretation.}
Define $L_{\mathbf{q}}(\mathbf{Z})=\frac{1}{m}\|\mathbf{Z}\mathbf{q}\|_2^2=\frac{1}{m}\mathbf{q}^\top \mathbf{Z}^\top \mathbf{Z}\mathbf{q}$, let $\mathbf{v}_1$ be a
top right singular vector of $\mathbf{Z}$, and set $\rho_1(\mathbf{Z})=\sigma_1(\mathbf{Z})^2/m$ and
$\gamma_{\mathbf{q}}=\langle \mathbf{q},\mathbf{v}_1\rangle^2$. Then
\begin{equation}
    \gamma_{\mathbf{q}}\rho_1(\mathbf{Z})\leq L_{\mathbf{q}}(\mathbf{Z})\leq\rho_1(\mathbf{Z}),
\end{equation}
where the upper bound is the Rayleigh quotient bound and the lower bound follows
by expanding $\mathbf{q}$ in the right-singular basis and retaining the $\mathbf{v}_1$ term. When
the proxy $\mathbf{q}(\mathbf{x})$ aligns with $\mathbf{v}_1$, the loss suppresses the leading spectral
component; when alignment is weak, the loss still penalizes the mean shared
direction but should not be read as a complete surrogate for the full spectrum.

% Temporarily hidden: extended related anti-collapse comparison.
% \input{appendix/tables/related_anti_collapse}

\paragraph{Spiked model: mean direction equals top singular direction.}
Let $\mathbf{z}_i=a_i \mathbf{u}+\boldsymbol{\xi}_i$, with $\mathbf{u}$ a unit shared direction, $\mathbf{u}^\top\boldsymbol{\xi}_i=0$,
$\mathbb{E}[\boldsymbol{\xi}_i]=0$, and isotropic residual covariance in the orthogonal
complement. Then $\mathbb{E}[\bar{\mathbf{z}}]=\mathbb{E}[a_i]\mathbf{u}$, and
\begin{equation}
\begin{split}
\mathbb{E}\!\left[\tfrac{1}{m}\mathbf{Z}^\top \mathbf{Z}\right]
={}& \mathbb{E}[a_i^2]\mathbf{u}\mathbf{u}^\top
   + \frac{\tau^2}{d-1}(\mathbf{I}-\mathbf{u}\mathbf{u}^\top).
\end{split}
\end{equation}
Whenever $\mathbb{E}[a_i^2]>\tau^2/(d-1)$, the population covariance has a
leading spike along $\mathbf{u}$, and the normalized mean direction and the leading right
singular direction identify the same shared component. This is why a lightweight
mean-direction proxy can replace an SVD inside the training loop.

\subsection{Signed-mean cancellation as a boundary case}
\label{app:proxy-audit}

A natural concern is a theoretical boundary case in which opposite directions cancel in the signed mean despite a rank-one trajectory. Let $\mathbf{z}_i=s_i\mathbf{u}$, where $s_i\in\{-1,+1\}$ and $\mathbf{u}$ is a unit vector. Since $\mathbf{Z}=\mathbf{s}\mathbf{u}^{\top}$, the trajectory remains rank one and SVD recovers $\mathbf{u}$ up to sign. For even $m$, however, equal numbers of positive and negative signs give $\sum_i\mathbf{z}_i=\mathbf{0}$; the underlying unit mean direction is then undefined, and the $\varepsilon$-stabilized proxy in Eq.~\eqref{eq:q-proxy} becomes the zero vector. With small off-axis perturbations, near balance can instead make the mean direction dominated by the perturbation residual and therefore unrepresentative of the leading rank-one axis.

This construction requires exact sign symmetry and is therefore an idealized boundary case; whether an approximate version occurs in practical trajectories is an empirical question. To distinguish spectral concentration from signed-mean cancellation, we use two quantities. First, $\rho_1=\sigma_1(\mathbf{Z})^2/m$ is the leading singular energy per state; for exactly unit-normalized rows, it equals the fraction of total trajectory energy captured by the leading singular direction, so $\rho_1$ near one indicates near-rank-one structure. Second, $R=\|m^{-1}\sum_i\mathbf{z}_i\|_2$ is the magnitude of the signed mean, so $R$ near zero indicates strong directional cancellation. We audit all $1{,}038$ ASDiv-Aug trajectories from the same LLaMA-3B + SIM-CoT checkpoint pair used in Section~\ref{sec:mechanism}. The audit uses all row-normalized states seen by the training loss.

We operationally flag near-rank-one cancellation when $\rho_1\geq0.9$ and $R\leq0.1195$: the first condition requires at least approximately $90\%$ of the trajectory energy to lie in one direction, whereas the second requires the signed mean to be unusually small. For seven exactly unit-normalized rows, $R^2=(1+6\bar c)/7$, where $\bar c$ is the average signed cosine over the $21$ unordered state pairs; hence $R\leq0.1195$ is approximately equivalent to $\bar c\leq-0.15$, or $90\%$ of the way from zero correlation to the general zero-mean limit $-1/6$. No trajectory qualifies in either the baseline checkpoint or the trained \shortmethod{} checkpoint ($0/1{,}038$ in each). Under this criterion, no near-rank-one cancellation failure is observed among the audited trajectories.
\afterpage{\input{appendix/tables/dataset_examples}}

% Temporarily hidden: Appendix B, Related Anti-Collapse Methods.
% \section{Related Anti-Collapse Methods}
% \label{app:related-anti-collapse}
%
% Section~\ref{sec:related-work} situates \shortmethod{} within
% representation-collapse methods at the family level. Table~\ref{tab:related}
% gives the extended comparison: for each family we list its main anti-collapse
% mechanism, the primary theoretical lens used in the literature, and how it
% relates to latent rank collapse in implicit reasoning models. The table is a
% positioning aid for readers familiar with self-supervised or
% representation-learning anti-collapse methods; it is not a benchmark
% comparison.

\section{Datasets}
\label{app:datasets}

Table~\ref{tab:train_datasets} summarizes the train/test sizes of the four
reasoning datasets used in Section~\ref{sec:experiments}, and
Table~\ref{tab:dataset_examples} shows one representative example per dataset
to illustrate the input--output formats. The same splits are used across all
backbones and method variants. Table~\ref{tab:asset_licenses} lists these
datasets together with the pre-trained models used in this work and their
licenses or terms of use; we use each asset in a manner consistent with its
license, restrict usage to non-commercial academic research where required,
and do not redistribute any underlying asset.

\input{appendix/tables/train_datasets}

\begin{table}[t]
\centering
\small
\caption{Existing assets used in this work and their licenses or terms of use.}
\label{tab:asset_licenses}
\begin{tabularx}{\linewidth}{@{}lX@{}}
\toprule
Asset & License / Terms \\
\midrule
StrategyQA & MIT License \\
CommonsenseQA & MIT License \\
ASDiv-Aug & CC BY-NC-ND 4.0 for the ASDiv-Aug release; CC BY-NC 4.0 for the original ASDiv dataset \\
AQuA & Apache-2.0 License \\
LLaMA-3.2-3B-Instruct & Llama 3.2 Community License and Acceptable Use Policy \\
LLaMA-3.1-8B & Llama 3.1 Community License and Acceptable Use Policy \\
Qwen3-1.7B & Apache-2.0 License \\
\bottomrule
\end{tabularx}
\end{table}

\section{Backbone and Training Configuration}
\label{app:backbone-config}
\label{app:implementation-details}

The main experiments are organized as matched comparisons: within each CODI or
SIM-CoT pair the backbone, dataset split, prompt template, parser, latent-step
schedule, training data, and checkpoint-selection rule are held fixed, and the
only intended difference is whether the spectral auxiliary loss in
Eq.~\eqref{eq:full-training-objective} is added. The no-CoT and SFT-CoT rows in
Table~\ref{tab:main_results} are contextual references, not ablations of the
spectral objective. We use \textbf{LLaMA-3.2-3B-Instruct} as the default
backbone for the main matched comparisons, the coefficient ablation, and the
ASDiv-Aug mechanism analysis, and \textbf{Qwen3-1.7B} as a second backbone
family and scale. For each backbone--dataset setting, the same evaluation
parser is used across rows so that accuracy differences are not caused by
answer extraction. The spectral coefficient is selected from the sweep in
Section~\ref{sec:ablation}. The auxiliary loss is active only during training;
at inference, the latent schedule and decoding procedure are unchanged, so
token counts in Table~\ref{tab:main_results} reflect the underlying reasoning
configuration rather than an additional spectral module. The concrete training
recipe below records the optimizer, schedule, and batching shared by each
baseline and its \shortmethod{} variant, so that the matched comparisons in
Section~\ref{sec:experiments} can be reproduced end-to-end.

\paragraph{CODI Training Setup.}
For CODI, we follow the configuration in \citet{shen2025codi}. The effective
batch size is set to $128$, and all experiments are conducted using
\texttt{bfloat16} precision. We train for $10$ epochs. The learning rate is
set according to the backbone: $1\mathrm{e}{-}5$ for LLaMA-3.2-3B-Instruct
and $1\mathrm{e}{-}4$ for Qwen3-1.7B.

\paragraph{SIM-CoT Training Setup.}
We reproduce SIM-CoT by implementing its method within the CODI
framework, keeping the effective batch size, precision, backbone-specific
learning rates, learning-rate schedule, and number of epochs identical to the
CODI setup above.

\paragraph{Hardware Information.}
We perform all experiments on multiple machines equipped with NVIDIA H800
80 GB GPUs. The software stack uses CUDA version $12.8$.

\section{Results on LLaMA-3.1-8B}
\label{app:llama31-8b}

We extend our experiments to LLaMA-3.1-8B using CODI as the implicit-CoT
baseline and the same matched training and evaluation protocol. As shown in
Table~\ref{tab:llama31-8b-accuracy}, \shortmethod{} improves accuracy on all
four benchmarks, with gains ranging from $0.32$ to $4.05$ percentage points.

\begin{table}[t]
\centering
\small
\begin{tabular}{lrrr}
\toprule
\textbf{Dataset} & \textbf{CODI} & \textbf{+\shortmethod{}} & \textbf{$\Delta$} \\
\midrule
StrategyQA    & 77.94 & \textbf{80.51} & +2.57 \\
CommonsenseQA & 82.56 & \textbf{82.88} & +0.32 \\
ASDiv-Aug     & 85.55 & \textbf{89.60} & +4.05 \\
AQuA          & 27.17 & \textbf{28.35} & +1.18 \\
\bottomrule
\end{tabular}
\caption{Accuracy (\%) on LLaMA-3.1-8B with CODI as the matched implicit-CoT
baseline. $\Delta$ denotes the absolute gain in percentage points.}
\label{tab:llama31-8b-accuracy}
\end{table}

We further test whether the same geometric pattern appears in this larger
backbone by computing the effective rank of the first six latent states.
Table~\ref{tab:llama31-8b-erank} shows that \shortmethod{} substantially raises
effective rank on every dataset.

\begin{table}[t]
\centering
\small
\begin{tabular}{lrr}
\toprule
\textbf{Dataset} & \textbf{CODI $\erank$} & \textbf{+\shortmethod{} $\erank$} \\
\midrule
StrategyQA    & 1.774 & \textbf{4.057} \\
CommonsenseQA & 2.092 & \textbf{5.523} \\
ASDiv-Aug     & 1.715 & \textbf{3.999} \\
AQuA          & 1.960 & \textbf{4.771} \\
\bottomrule
\end{tabular}
\caption{Effective rank of the first six latent states on LLaMA-3.1-8B.}
\label{tab:llama31-8b-erank}
\end{table}

These results show that latent rank collapse persists beyond compact models.
A larger backbone provides a wider representation space, but successive latent
reasoning states can still concentrate along only a few directions. The higher
effective ranks under \shortmethod{} indicate broader trajectory subspaces,
rather than an effect obtained merely by adding model capacity. This finding
supports our original analysis: latent rank collapse arises when successive
latent states align with a shared dominant direction, rather than merely when
a model has too few parameters.

\section{Dominant-Direction Vocabulary Projection Protocol}
\label{app:vocab-projection}

For each example we stack the first six implicit states into
$\mathbf{H}(\mathbf{x})\in\mathbb{R}^{6\times d}$, take its first right singular vector $\mathbf{v}_1(\mathbf{x})$,
and score the vocabulary by $\mathbf{s}(\mathbf{x})=\mathbf{W}_{\mathrm{LM}}\mathbf{v}_1(\mathbf{x})$. Since the sign of
$\mathbf{v}_1(\mathbf{x})$ is arbitrary, we rank by $|\mathbf{s}(\mathbf{x})|$ and keep the top $N=30$ tokens per
sample. Each token is deterministically assigned to one of five surface-form
categories---\emph{punctuation}, \emph{function word} (closed-class English),
\emph{number} (pure numeric literal), \emph{whitespace/special} (whitespace,
byte-level artifacts, control tokens), and \emph{content-like} (everything
else)---using the same categorizer for baseline and Spectral models. We report
per-sample slot shares averaged across samples, and summarize via
$\mathrm{ContentShare}$ (the content-like share) and
$\mathrm{ShortcutShare}=1-\mathrm{ContentShare}$, where shortcut share combines
punctuation, function-word, number, and whitespace/special shares, so the two
summaries sum to $100\%$. The diagnostic indicates which vocabulary region the dominant
singular direction occupies; it is not a claim that individual top-projected
tokens are semantically meaningful.

\paragraph{Aggregation and reporting.}
For every checkpoint we compute the five surface-form shares on a held-out
subset of evaluation prompts, then average across prompts to obtain a single
per-checkpoint share vector. Reported figures pair each baseline with its
matched \shortmethod{} variant on the same backbone, dataset, and
evaluation subset, so the absolute share values are comparable across rows.
We do not threshold or filter tokens beyond the top-$N$ rank, and we do not
re-weight categories; differences between paired rows therefore reflect a
direct shift in where the dominant singular direction concentrates, rather
than any post-hoc selection of which tokens to count.

\paragraph{Scope and caveats.}
The protocol is a diagnostic of where the dominant latent direction
projects in the vocabulary, not a token-level interpretation. Two cautions
apply. First, the sign of $\mathbf{v}_1(\mathbf{x})$ is arbitrary, so the
projection magnitude $|\mathbf{s}(\mathbf{x})|$ is the only meaningful
quantity at the token level; individual rankings are not invariant under
unit-norm reparameterizations of nearby singular directions. Second, the
surface-form categorizer is deterministic but coarse: it does not
distinguish between, e.g., content-like tokens that are stop-word
collocations and those that carry task-specific signal. We therefore
interpret the diagnostic as a population-level statement about which
vocabulary region absorbs latent energy, rather than as an account of which
specific tokens drive any given prediction.

\paragraph{Implementation and comparability.}
We compute $\mathbf{s}(\mathbf{x})$ in full precision after casting the
latent states and the unembedding $\mathbf{W}_{\mathrm{LM}}$ to
\texttt{float32}, so that ranking decisions do not depend on the
\texttt{bfloat16} training precision used elsewhere in the pipeline. The
top-$N$ cutoff is fixed at $N{=}30$ across all models and datasets, and the
same prompt subset is reused for every checkpoint of a given
backbone--dataset pair, so paired rows in the main tables differ only in
which model produced the latent states rather than in which prompts or
ranks were inspected. When a token spans multiple byte-pair pieces, each
piece is categorized independently and contributes to the slot share of its
own surface form; we do not merge subword pieces back into whole words,
because doing so would require a tokenizer-specific reassembly step that
behaves differently across backbones and would obscure the
backbone-agnostic comparison we report. The full categorizer, including
the regular expressions used to detect numeric literals and the closed
list of function words, is released with the code so that all share values
in the main text can be recomputed from a released checkpoint without
re-implementing any heuristics.

%% file: appendix/tables/dataset_examples.tex
\begin{table*}[!t]
\centering
\small
\setlength{\tabcolsep}{5pt}
\renewcommand{\arraystretch}{1.15}
\begin{tabular}{@{}>{\raggedright\arraybackslash}p{0.11\textwidth} >{\raggedright\arraybackslash\hyphenpenalty=50\relax}p{0.34\textwidth} >{\centering\arraybackslash}p{0.04\textwidth} >{\raggedright\arraybackslash}p{0.44\textwidth}@{}}
\toprule
\textbf{Dataset} & \textbf{Question} & \textbf{Ans.} & \textbf{CoT} \\
\midrule
StrategyQA &
Are more people today related to Genghis Khan than Julius Caesar? &
True &
(1) Genghis Khan is known to have had many children and thousands of descendants across Asia.
(2) Genetic studies suggest millions of men carry a Y-chromosome traceable to him.
(3) Julius Caesar does not have a similarly documented wide-spanning direct lineage.
(4) Historical and genetic evidence does not indicate a comparable descendant impact for Caesar. \\
\midrule
Commonsense-QA &
The sanctions against the school were a punishing blow, and they seemed to \underline{\hspace{1.2em}} what efforts the school had made to change?
Choices: (A) ignore (B) enforce (C) authoritarian (D) yell at (E) avoid &
A &
The sanctions are undermining or dismissing the school's efforts. ``Ignore'' fits best, conveying that the sanctions do not acknowledge those efforts. \\
\midrule
ASDiv-Aug &
gino has 64 popsicle sticks. i have 100 popsicle sticks. what is the sum of our popsicle sticks? &
164 &
$\langle\langle 64+100=164 \rangle\rangle$ \\
\midrule
AQuA &
Two friends walk along a 43-km trail, starting at opposite ends at the same time. If P's rate is 15\% faster than Q's, how many km will P have walked when they pass?
Choices: (A) 21 (B) 21.5 (C) 22 (D) 22.5 (E) 23 &
E &
If Q walks $x$ km, P walks $1.15x$ km. Then $x+1.15x=43\Rightarrow 2.15x=43\Rightarrow x=20$, so P walks $1.15\times20=23$ km. The answer is E. \\
\bottomrule
\end{tabular}
\caption{Representative examples from each training dataset, illustrating the input--output formats spanned by the four benchmarks.}
\label{tab:dataset_examples}
\end{table*}

%% file: appendix/tables/train_datasets.tex
\begin{center}
\small
\setlength{\tabcolsep}{6pt}
\renewcommand{\arraystretch}{1.1}
\begin{tabular}{lrr}
\toprule
\textbf{Dataset} & \textbf{Train} & \textbf{Test} \\
\midrule
StrategyQA      & 1{,}537  & 272     \\
CommonsenseQA   & 8{,}196  & 1{,}221 \\
ASDiv-Aug       & 4{,}183  & 1{,}038 \\
AQuA            & 97{,}467 & 254     \\
\bottomrule
\end{tabular}
\captionof{table}{Training and test sizes for the four reasoning datasets.}
\label{tab:train_datasets}
\end{center}